\documentclass[twoside,11pt]{article}

\usepackage{jmlr2e}

\usepackage{graphicx}

\usepackage{algorithm}
\usepackage{algorithmic}

\usepackage{amsmath}
\usepackage{amsthm}
\usepackage{verbatim}
\newtheorem{theo}{Theorem}
\newtheorem{lem}{Lemma}

\newtheorem{defi}{Definition}
  
\usepackage{url}            
\usepackage{amsfonts}
\usepackage{epstopdf}
\usepackage{appendix}

\usepackage{geometry}

\jmlrheading{1}{2021}{1-29}{1/21}{10/00}{meila00a}{Wei Chen, Kun Zhang, Ruichu Cai, Biwei Huang, Joseph Ramsey, Zhifeng Hao, and Clark Glymour}

\ShortHeadings{FRITL: A Hybrid Method for Causal Discovery in the Presence of Latent Confounders}{Chen, Zhang, Cai, Huang, Ramsey, Hao, and Glymour}
\firstpageno{1}

\begin{document}

\title{FRITL: A Hybrid Method for Causal Discovery in the Presence of Latent Confounders}

\author{\name Wei Chen \email chenweiDelight@gmail.com  \\
\addr School of Computer Science, Guangdong University of Technology, Guangzhou, 510006, China 
\AND
       \name Kun Zhang \textsuperscript{*} \email kunz1@cmu.edu \\
       \addr Department of Philosophy, Carnegie Mellon University, Pittsburgh, PA 15213, USA
      \AND
       \name Ruichu Cai \textsuperscript{*} \email cairuichu@gmail.com \\
       \addr School of Computer Science, Guangdong University of Technology, Guangzhou, 510006, China\\
       Pazhou Lab, Guangzhou, 510330, China
       \AND
       \name Biwei Huang \email  biweih@andrew.cmu.edu \\
       \addr Department of Philosophy, Carnegie Mellon University, Pittsburgh, PA 15213, USA
       \AND
       \name Joseph Ramsey \email   jdramsey@andrew.cmu.edu \\
       \addr Department of Philosophy, Carnegie Mellon University, Pittsburgh, PA 15213, USA
       \AND
       \name Zhifeng Hao \email   zfhao@gdut.edu.cn \\ 
       \addr School of Mathematics and Big Data, Foshan University, Foshan, 528000, China 
       \AND
       \name Clark Glymour \email cg09@andrew.cmu.edu \\
       \addr Department of Philosophy, Carnegie Mellon University, Pittsburgh, PA 15213, USA}

\editor{ }

\maketitle

\begin{abstract}
We consider the problem of estimating a particular type of linear non-Gaussian model. Without resorting to the overcomplete Independent Component Analysis (ICA), we show that under some mild assumptions, the model is uniquely identified by a hybrid method. Our method leverages the advantages of constraint-based methods and independent noise-based methods to handle both confounded and unconfounded situations. The first step of our method uses the FCI procedure, which allows confounders and is able to produce asymptotically correct results. The results, unfortunately, usually determine very few unconfounded direct causal relations, because whenever it is possible to have a confounder, it will indicate it. The second step of our procedure finds the unconfounded causal edges between observed variables among only those adjacent pairs informed by the FCI results. By making use of the so-called Triad condition, the third step is able to find confounders and their causal relations with other variables. Afterward, we apply ICA on a notably smaller set of graphs to identify remaining causal relationships if needed. Extensive experiments on simulated data and real-world data validate the correctness and effectiveness of the proposed method.
\end{abstract}


\section{Introduction}

Causal discovery is crucial for understanding the actual mechanism underlying events in fields such as neuroscience \cite{sanchez2019estimating}, biology \cite{sachs2005causal} and social networks \cite{cai2016understanding}. In such areas, the aim of the inquiry is to discover causal relations among variables that are measured only indirectly. Unmeasured variables and their influence on measured variables are unknown prior to the inquiry. Various methods for discovering the causal structure from observed samples have been proposed. However, most of them assume that the system of variables is causal sufficient, which means no pairs of variables have an unmeasured common cause (also called a latent confounder) \cite{spirtes2001causation}. Real applications typically violate this assumption. For example, some variables might not be measured because of limitations in data collection, and other variables may not even be considered in the data collection design. Without considering the presence of latent confounders, these algorithms return some false causal relations. Thus, developing a causal discovery method in the presence of latent confounders is an important research topic.

Methods for finding latent confounders and their relationships began early in the 20th century in factor analysis and its applications. In the case of continuous variables, linear relationships among variables are widely used as the data-generation assumption in searches for structural equation models (SEMs). Recently SEMs have begun to employ non-Gaussian additive (unmeasured) disturbances for each variable. The LvLiNGAM \cite{hoyer2008estimation} algorithm, which uses overcomplete Independent Components Analysis (ICA) \cite{eriksson2004identifiability} \cite{lewicki2000learning}, has been proposed to estimate the causal relations among measured variables in systems with linearly related variables. Given the number of latent confounders and appropriate data, it can in principle identify the measured variables sharing a common cause or causes, as well as the causal relations between measured variables, but it requires latent confounders to be mutually independent. This independence is impractical when the number of variables is large. The algorithm easily falls into local optima, which produces estimation errors aggravated by high-dimensional data. The ParceLiNGAM \cite{tashiro2014parcelingam} and PairwiseLvLiNGAM \cite{entner2010discovering} methods have been proposed for the same model class, but these methods fail to identify the causal structure given in Fig. \ref{figure.exp1}. Existing independence noise-based methods have a high computational load and do not fully identify the causal structure. 

\begin{figure}[ht]
	\centering
	\includegraphics[width=0.18\textwidth]{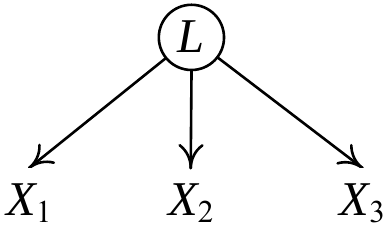} 
	\caption{An example of a causal graph, where $X_1$, $X_2$ and $X_3$ are observed variables, and $L$ is a latent confounder.} \label{figure.exp1}
\end{figure}

Constraint-based methods such as the Fast Causal Inference (FCI) algorithm \cite{spirtes2001causation} is another type of methods for recovering causal structures. Although the results of the FCI algorithm are statistically consistent but provide limited information. For example, even when no confounders exist, FCI usually provides too few directed, unconfounded causal relationships; on the other hand, for a small number of variable pairs, hidden variables usually can not be found. As a specific example, consider the data generated according to the Directed Acyclic Graph (DAG) shown in Figure \ref{figure.pag}(a). The FCI output, called Partial Ancestral Graph (PAG), is given as Figure \ref{figure.pag}(b). The adjacency and arrowheads in Figure \ref{figure.pag}(b) are mostly correct, but some undetermined tails of edges remain.

From these observations, we propose a hybrid method assuming linearity and non-Gaussianity, to take advantages of both constraint-based methods and independent noise-based methods to handle both confounded and unconfounded situations. However, designing such a solution is a non-trivial task due to the two specific challenges raised by the high dimensionality of the measured variables and the latent confounders. One is how to efficiently decompose a large global graph into local small structures without introducing new latent confounders. The second is how to recover local structures accurately in the presence of latent confounders. To address these challenges, we first employ FCI to remove some independent causal relationships. This output will not be complete, in the sense that it contains many undetermined causal edges when latent confounders might not exist. We further refine this output to examine unconfounded causal edges and locate the latent confounders by applying an independent noise-based method among only those adjacent pairs informed by the FCI result. The Triad condition \cite{cai2019Triad} identifies some shared latent confounders and the causal relations between measured variables. If some causal directions are still undetermined, we apply overcomplete ICA locally to refine the causal structure.

\begin{figure}[ht]
	\centering
	\includegraphics[width=0.6\textwidth]{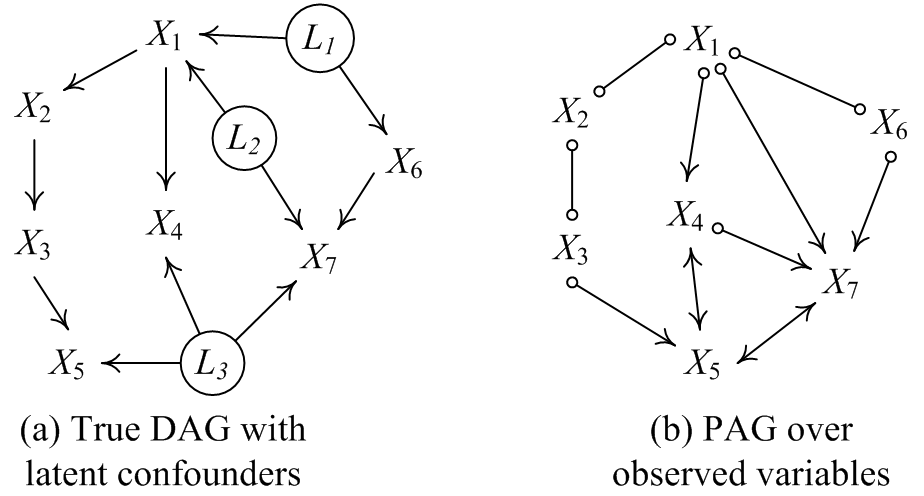} 
	\caption{An example for graphs of observed variables $X_{1},X_{2},\dots, X_{7}$ and latent confounders $L_{1}$, $L_{2}$ and $L_{3}$: (a) the original directed acyclic graph (DAG), and (b) the corresponding PAG produced by FCI.} \label{figure.pag}
\end{figure}

We summarize our contributions as follows:
\begin{itemize}
    \item [1)]  We propose a hybrid framework to reconstruct the entire causal structure from measured data, handling both confounded and unconfounded situations. 
    \item [2)] We show the completeness result of our proposed method, demonstrating the correctness of our method on the theoretical side; 
    \item [3)] We verify the correctness and effectiveness of our method on simulated and real-world data, showing results to be mostly consistent with the background knowledge.
\end{itemize}


\section{Graphical Models}
We employ two types of graphical representations of causal relations: Directed Acyclic Graphs (DAGs) and Partial Ancestral Graphs (PAGs).

\subsection{DAG Description}
A DAG can be used to represent both causal and independence relationships. A DAG contains a set of vertices and a set of directed edges ($\to$), where each vertex represents one random variable. $X_{i} \to X_{j}$ means that $X_{i}$ is a ``direct" cause (or parent) of $X_{j}$, that is, $X_{j}$ is a direct effect (or child) of $X_{i}$. Figure \ref{figure.pag}(a) shows an example of a DAG $\mathcal{G}$. In Figure \ref{figure.pag}(a), $X_{1}$ is a parent of $X_{2}$, or $X_{2}$ is a child of $X_{1}$, due to the edge $X_{1} \to X_{2}$. Two vertices $(X_{i}, X_{j})$ are adjacent if there is a directed edge $X_{i} \to X_{j}$ or $X_{j}\to X_{i}$. A directed path from $X_{i}$ to $X_{j}$ is a sequence of vertices beginning with $X_{i}$ and ending with $X_{j}$ such that each vertex in the sequence is a child of its predecessor in the sequence. Any sequence of vertices in which each vertex is adjacent to its predecessor is an undirected path. A vertex, in a path $X_{i}$ is a collider if $X_{i}$ is a child of both its predecessor and its successor in the path. 

\textbf{d-separation \cite{pearl1988probabilistic}.} Let $\mathbf{X}$ be a set of variables in DAG $\mathcal{G}$ that does not have either of $X_{i}$ and $X_{j}$ as members. $X_{i}$ and $X_{j}$ are d-separated given $\mathbf{X}$ if and only if there exists no undirected path $U$ between $X_{i}$ and $X_{j}$, such that both of the following conditions hold:

(i) every collider on $U$ has a descendent (or itself) in $\mathbf{X}$;

(ii) no variable on U that is not a collider is in $\mathbf{X}$.

Two variables that are not d-separated by $\mathbf{X}$ are said to be d-connected given $\mathbf{X}$.

\subsection{PAG Description}

 A PAG contains four different types of edges between two variables: directed edge ($\to$), bidirected edge ($\leftrightarrow$), partially directed edge ($\circ\!\!\to$), and nondirected edge ($\circ\!\!-\!\!\circ$). A directed edge $X_{i} \to X_{j}$ means that $X_{i}$ is a cause of $X_{j}$. A bidirected edge $X_{i} \leftrightarrow X_{j}$ indicates that there is a latent confounder that is a common cause of $X_{i}$ and $X_{j}$. A partially directed edge $X_{i} \circ\!\!\to X_{j}$ indicates that either $X_{i}$ is a cause of $X_{j}$, or there is an unmeasured variable influencing $X_{i}$ and $X_{j}$, or both. A nondirected edge $X_{i} \circ\!\! -\!\!\circ X_{j}$ means exactly one of the following holds: (a) $X_{i}$ is a cause of $X_{j}$; (b) $X_{j}$ is a cause of $X_{i}$; (c) there is an unmeasured variable influencing $X_{i}$ and $X_{j}$; (d) both a and c; or (e) both b and c. In a PAG, the end marks of some edges may be undetermined, i.e., the undetermined edge is an edge other than the directed edge.
 
 
Figure \ref{figure.pag}(b) shows a PAG representing the set of all DAGs that imply the same conditional independence relations \textit{among the measured variables} as does the DAG $\mathcal{G}$ (Figure \ref{figure.pag}(a)). For example, the bidirected edge between $X_{4}$ and $X_{5}$ means that there is a latent confounder influencing $X_{4}$ and $X_{5}$. The non-directed edge between $X_{1}$ and $X_{2}$ shows a class of causal relation between $X_{1}$ and $X_{2}$, that is, this edge might be: $X_{1} \to X_{2}$, $X_{1} \gets X_{2}$, $X_{1} \leftrightarrow X_{2}$. 
 
A PAG can be estimated by the FCI algorithm \cite{spirtes2001causation} or one of its variants such as RFCI \cite{colombo2012learning}, FCI$^{+}$ \cite{claassen2013learning}, FCI-stable \cite{colombo2014order}, Conservative FCI (CFCI) \cite{ramsey2006adjacency} or Greedy FCI (GFCI) \cite{ogarrio2016hybrid}.

\section{Problem Definition}

To help with the definition of the scope of our solution, we assume that all samples are infinite, independent distributed, following the same joint probability distribution $P$. Further, we make some or all of the following assumptions according to context.

$\mathbf{A1}$. \textbf{Causal Markov Assumption}. Two variables $X_{i}$ and $X_{j}$ are independent given a subset $\mathbf{S}$ of variables not containing $X_{i}$ and $X_{j}$, if $X_{i}$ and $X_{j}$ are d-separated given $\mathbf{S}$.  

$\mathbf{A2}$. \textbf{Causal Faithfulness Assumption}. Let $\mathbf{S} \bigcap \{X_{i}, X_{j} \} = \emptyset$. If $X_{i}$ and $X_{j}$ are independent conditional on $\mathbf{S}$ in $P$, then $X_{i}$ is d-separated from $X_{j}$ conditional on $\mathbf{S}$ in $\mathcal{G}$.

We assume the target to be discovered is a DAG, represented as a linear non-Gaussian model with latent confounders (named as LvLiNGAM), as defined by Hoyer et al. \cite{hoyer2008estimation}, in which each measured variable $X_{i}$ in $\mathbf{X}$, $i=1,2,\dots,n$, is generated from its parents including measured variables and latent confounders $\mathbf{L}$ with an additive noise term. The matrix form of LvLiNGAM then can be formalized as
\begin{equation}
\label{eq:model}
\mathbf{X}=\mathbf{BX}+\Lambda \mathbf{L}+\mathbf{E},
\end{equation}
where $\mathbf{B}$ is the matrix of causal strengths among measured variables, $\Lambda$ is the matrix of causal influences of the latent confounders $\mathbf{L}$ on measured variables, and the noise terms, as components of $\mathbf{E}$, are mutually independent and non-Gaussian. According to lvLiNGAM by Hoyer et al. \cite{hoyer2008estimation}, we know that $\mathbf{B}$ can be permutated to be a lower triangular matrix, and Equation \ref{eq:model} can be changed to
\begin{equation}
\begin{aligned}
&\mathbf{X} =\mathbf{BX}+\Lambda \mathbf{L}+\mathbf{E},\\
\Rightarrow & \mathbf{X} = (\mathbf{I-B})^{-1}\Lambda \mathbf{L}+ (\mathbf{I-B})^{-1}\mathbf{E},\\
\Rightarrow & \mathbf{X} = [\begin{array}{c}
(\mathbf{I-B})^{-1}\Lambda \mid \mathbf{I} 
\end{array}] \cdot \left[\begin{array}{c}
\mathbf{L}\\ \mathbf{E}
\end{array} \right]= \mathbf{A} \cdot \left[\begin{array}{c}
\mathbf{L}\\ \mathbf{E}
\end{array} \right],
\end{aligned}
\label{eq:model2}
\end{equation}
where $\mathbf{A} := [\begin{array}{c}
(\mathbf{I-B})^{-1}\Lambda \mid \mathbf{I} 
\end{array}] $ and $\mathbf{I}$ denotes the identity matrix.

Based on Equation \ref{eq:model}, we make the following further assumptions:

$\mathbf{A3}$. \textbf{Linear Acyclic Non-Gaussianity Assumption}. The causal graph over all variables, including the latent variables, is a directed acyclic graph (DAG), which represents the model in which the causal relations among any variables are linear and all noise terms are non-Gaussian and mutually independent.

$\mathbf{A4}$. \textbf{One Latent Confounder Assumption}. All latent confounders are independent of each other, and each pair of observed variables is directly influenced by at most one latent confounder.

Based on the above assumptions, we define our problem as follows.
\begin{defi} 
   \textbf{(Problem Definition)} Given the observational data generated by causal model as Equation \ref{eq:model}, reconstruct the causal graph over measured variables and latent confounders.
\end{defi}

\section{A Hybrid Method for Causal Discovery in the Presence of Latent Confounders} 

In this section, we describe our approach in detail, and  explaining how it recover the true graph shown in Figure \ref{figure.f1}(a) that represents causal model (\ref{eq:model}). The proposed framework is given in Figure \ref{figure.f1}.

\begin{figure}[ht]
	\centering
	\includegraphics[width=0.9\textwidth]{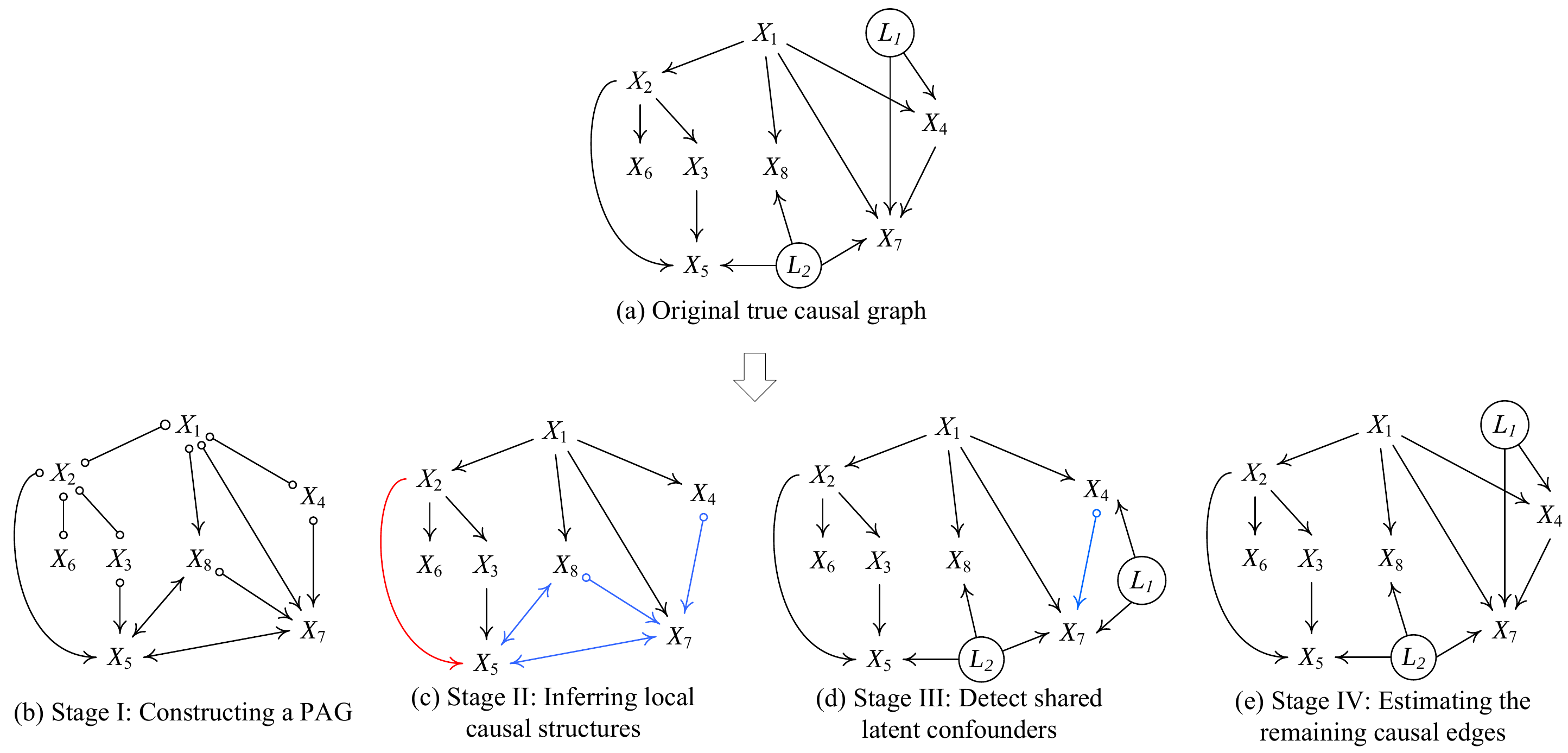} 
	\caption{The proposed framework. In these graphs, $X_{1},X_{2},\dots, X_{8}$ represent the measured variables and $L_{1}$ and $L_{2}$ represent the latent confounders. The red line in (c) means the edge $X_2 \to X_5$ can be updated by FCI orientation rule $\mathcal{R}_8$. Blue lines in (c) and (d) indicate the edges for which end marks are not determined.} \label{figure.f1}
\end{figure}

The idea is as follows. After running the FCI algorithm to obtain a PAG, we further try to orient edges by regression and subsequent independence testing, extrapolating directions by the well-known Meek rules. From regression residuals, we further determine local causal structures for pairs of variables that are adjacent with an undetermined edge in the PAG. We then introduce a constraint condition for triples of variables to detect and combine some latent confounders. Finally, under the further assumption that the latent confounders are independent, we use overcomplete ICA to determine the remaining edges when needed. The pseudo-code of this framework (named FRITL) is described in Algorithm \ref{alg:1}. The use of these four steps can be selected according to the purpose. 


\begin{algorithm}
	\caption{FRITL Algorithm}
	\label{alg:1}
	{\bf Input:} Data $\mathcal{D}$, threshold for independence test $\alpha$ \\
	{\bf Output:} Causal graph $\mathcal{G}_{out}$ over measured variables and latent confounders    
	\begin{algorithmic}		
		\STATE {\textbf{Stage I:} \textit{Construct a PAG} $\mathcal{G}_1$ by running the FCI algorithm on $\mathcal{D}$;}
		\STATE {\textbf{Stage II:} \textit{Infer local causal structures} by using an independence noise condition for undetermined adjacent pairs of variables in $\mathcal{G}_1$; update $\mathcal{G}_1$ to $\mathcal{G}_2$;}
		\STATE {\textbf{Stage III:} \textit{Detect shared latent confounders} by using Triad conditions and update $\mathcal{G}_2$ to $\mathcal{G}_3$;}
		\STATE {\textbf{Stage IV:} \textit{Estimate remaining undetermined local causal structures} in $\mathcal{G}_3$ using overcomplete ICA. Update $\mathcal{G}_3$ to $\mathcal{G}_{out}$.}
	\end{algorithmic}
\end{algorithm}

\subsection{Stage I: Constructing PAG Using FCI}
We begin by supposing that the data generated by causal model (\ref{eq:model}) satisfies assumptions \textbf{A1}-\textbf{A2}. The FCI algorithm outputs a PAG, which represents estimated features of the true causal DAG according to the following theorem \cite{spirtes2001causation} and lemma.

\begin{theo}
	Given the assumptions \textbf{A1}-\textbf{A2}, the FCI algorithm outputs a PAG that represents a class of graphs including the true causal DAG.
\end{theo}

\begin{lem}
    Given the assumptions \textbf{A1}-\textbf{A2}, if FCI converges to a PAG with a directed edge between $X_{i}$ and $X_{j}$, then there is a directed edge between $X_{i}$ and $X_{j}$ in the true DAG.
\end{lem}

We first apply FCI on the data to obtain a PAG. For example, using the graph representing a lvLiNGAM model in Figure \ref{figure.f1}(a), the FCI algorithm outputs the PAG shown in Figure \ref{figure.f1}(b).

\subsection{Stage II: Inferring Local Structures Using Independence Noise Condition}

After running stage I, we obtain the PAG (given in Figure \ref{figure.f1}(b)) that is (asymptotically) correct information of the causal structure but usually provides few direct influences. Although we can apply overcomplete ICA to estimate the true causal graph, the result may suffer from local optima, especially if the number of measured variables is larger than four. In contrast, the ``divide-and-conquer" provides more causal information about the undetermined edges in the graph and only requires performing overcomplete ICA on a small number of variables to estimate the local causal structures. This second stage produces correct, informative causal discovery result with relatively low computational complexity. We note that in the linear non-Gaussian case, unconfounded causal relations can always be determined by regression and independence testing \cite{shimizu2011directlingam}. Inspired by this, we consider generalizing regression and independence test from global causal structure to local causal structure, even when there are latent confounders.

\subsubsection{Identification of causal direction between unconfounded pairs of variables}

We first provide a lemma to identify the causal direction of variables that are not influenced by confounders. Let $\mathcal{G}$ denote the PAG obtained by FCI. From the definition of a PAG, variables connected to the measured variable $X_{i}$ through a directed, nondirected, or partially directed edge are the potential parents of $X_{i}$. For example, $X_{j}$ is a potential parent of $X_{i}$ if $X_{j} \to X_{i}$, $X_{j} \circ\!\! -\!\!\circ X_{i}$, or $X_{j} \circ\!\!\!\to X_{i}$. Let $X_{\mathbf{par}(i)}$ denote the potential parents of $X_{i}$ in $\mathcal{G}$.

If there are no latent or observed confounders of $X_{i}$ and any of $X_{\mathbf{par}(i)}$, we can generalize Lemma 1 proposed by Shimizu et al. \cite{shimizu2011directlingam} to determine local causal structures. We first introduce the Darmois-Skitovitch Theorem \cite{darmois1953analyse}\cite{skitovitch1953property}, which determines whether each potential parent is an actual parent of $X_{i}$.

\begin{theo}\textbf{(Darmois-Skitovitch Theorem).} Define two random variables $X_{1}$ and $X_{2}$, as linear combinations of independent random variables $S_i, i=1, \dots, n$:
	\begin{equation}
	X_{1} = \sum_{i=1}^{n}\alpha_{i}S_{i},
	X_{2} = \sum_{i=1}^{n}\beta_{i}S_{i}.
	\end{equation}
	
	If $X_{1}$ and $X_{2}$ are statistically independent, then all variables $S_j$ for which $\alpha_{j}\beta_{j} \neq 0$ are Gaussian.
	\label{theo.1}
\end{theo}
	
In other words, if random variables $S_i, i=1, \dots, n$ are independent and for some $\alpha_{1}, \alpha_{2},\dots, \alpha_{n}$ and $\beta_{1}, \beta_{2},\dots, \beta_{n}$, $X_1$ is independent of $X_2$, then for any $S_{j}$ that is non-Gaussian, at most one of $\alpha_{j}$ and $\beta_{j}$ can be nonzero.

\begin{lem}
	Suppose that the data over variables $\mathbf{X}$ are generated by (\ref{eq:model}) and that assumptions \textbf{A1}-\textbf{A3} hold. Assume there is no latent or observed confounder relative to $X_{i}$ and $X_{j}$ in the underlying true causal graph over all given variables, where $X_{j}$ is one of the potential parents of $X_{i}$ in the FCI output. Let $R_{i,j}$ be the residual of the regression of $X_{j}$ on $X_{i}$. Then in the limit of infinite data, $X_{i}$ is an unconfounded ancestor of $X_{j}$ if and only if $X_{i} \perp\!\!\!\perp R_{j,i}$ and $X_{j} \not\perp\!\!\!\perp R_{i,j}$.
	\label{pro:pair}
\end{lem}

\begin{proof}
Without loss of generality, all these data are normalized to have zero mean and unit variance.
 
1. Assume that $X_{i}$ is an ancestor of $X_{j}$ and that $X_{i}$ is an exogenous variable, which means that there are no parent or latent confounders for $X_{i}$ and $X_{j}$. $X_{i}$ and $X_{j}$ are generated by (\ref{eq:model}). This leads to
\begin{equation}
	\begin{aligned}
       X_{i} & = E_{i}, \\
       X_{j} & = b_{j,i} X_{i} + E_{j}^{(-i)},
    \end{aligned}
    \label{eq:Xj}
\end{equation}
where $E_{j}^{(-i)} = \sum_{k\in \mathbf{par}(j), k\neq i}b_{j,k}X_{k} + E_{j}$ and $X_{i}$ are independent. 

(1) The residual of regressing $X_{j}$ on $X_{i}$ will be
\begin{equation}
   \begin{aligned}  	
     R_{j,i} & = X_{j} - \frac{\mathrm{Cov}(X_{j},X_{i})}{Var(X_{i})}\cdot X_{i} \\
      & = (b_{j,i} X_{i} + E_{j}^{(-i)}) - b_{j,i}X_{i}\\
      & = E_{j}^{(-i)}.
   \end{aligned}
\end{equation}  

Thus, the residual $R_{j,i}$ is independent of $X_{i}$ because $E_{j}^{(-i)}$ is independent of $X_{i}$. 

(2) If instead we regress $X_{i}$ on $X_{j}$, the residual will be
\begin{equation}
   \begin{aligned}
   	     R_{i,j}  =& X_{i} - \frac{\mathrm{Cov}(X_{i},X_{j})}{Var(X_{j})} \cdot X_{j} \\
       =& E_{i} - \frac{\mathrm{Cov}(X_{i},X_{j})}{Var(X_{j})} \cdot (b_{j,i} X_{i} + E_{j}^{(-i)})\\
       =& (1 - \frac{\mathrm{Cov}(X_{i},X_{j})}{Var(X_{j})}) \cdot E_{i} - \frac{b_{j,i}\mathrm{Cov}(X_{i},X_{j})}{Var(X_{j})} \cdot \sum_{k\in \mathbf{par}(j), k\neq i}b_{j,k}X_{k}\\
      &- \frac{b_{j,i}\mathrm{Cov}(X_{i},X_{j})}{Var(X_{j})} E_{j}.
   \end{aligned}
   \label{eq:wrongR}
\end{equation}  

Each parent of $X_{j}$ is a linear mixture of error terms including $E_{j}$, where all the error terms are mutually independent and non-Gaussian according to assumption \textbf{A3}. Thus, the residual $R_{i,j}$ is a mixture of $E_{i}$, $E_{j}$, and $X_{k}(k\in \mathbf{par}(j), k\neq i)$, where each $X_{k}(k\in \mathbf{par}(j), k\neq i)$ is non-Gaussian. From Equations (\ref{eq:Xj}) and (\ref{eq:wrongR}), the coefficient of $E_{j}$ is non-zero, which implies that $X_{j}$ is dependent of $R_{i,j}$ according to Theorem \ref{theo.1}. Thus, if $X_{i}$ is an ancestor of $X_{j}$, then $X_{i}$ is dependent of $R_{j,i}$ and $X_{j}$ is dependent of $R_{i,j}$.

2. Assume that $X_{i}$ and $X_{j}$ have at least one common ancestor. Let $X_{\mathbf{Pa}_{i}}$ denote all parents of $X_{i}$, and $X_{k}$ be an actual parent of $X_{i}$. Then we have
\begin{equation}
X_{i} = \sum_{k \in \mathbf{Pa}_{i}} b_{i,k} X_{k} + E_{i}.
\label{eq:XiNotCause}
\end{equation}

If we regress $X_{j}$ on $X_{i}$, the residual $R_{j,i}$ will be
\begin{equation}
	\begin{aligned}
	R_{j,i}  =& X_{j} - \frac{Cov(X_{i},X_{j})}{Var(X_{i})} \cdot X_{i}\\
	=& X_{j} - \frac{Cov(X_{i},X_{j})}{Var(X_{i})} \cdot (\sum_{k \in \mathbf{Pa}_{i}} b_{i,k} X_{k} + E_{i})\\
	 =& (1-\frac{b_{i,j}Cov(X_{i},X_{j})}{Var(X_{i})})\cdot X_{j}- \frac{Cov(X_{i},X_{j})}{Var(X_{i})} \cdot \sum_{k \in \mathbf{Pa}_{i},k \neq j} b_{i,k} X_{k} \\
	&- \frac{Cov(X_{i},X_{j})}{Var(X_{i})} \cdot E_{i}.
	\end{aligned} 
	\label{eq:wrongRwithPa}
\end{equation}

Each parent of $X_{i}$ is a linear mixture of error terms other than $E_{i}$, with all the error terms mutually independent and non-Gaussian according to assumption \textbf{A3}. Thus, the residual $R_{j,i}$ can be written as a linear mixture of error terms including $E_{i}$. We can see that the coefficient of $E_{i}$ in Equations (\ref{eq:XiNotCause}) and (\ref{eq:wrongRwithPa}) is nonzero due to $Cov(X_{i},X_{j})\neq 0$, which implies that $X_{i}$ is dependent of $R_{j,i}$ according to Theorem \ref{theo.1}. 
\end{proof}

Lemma \ref{pro:pair} provides a principle to determine the causal direction between a pair of measured variables. If there is no latent or observed confounder for $X_{i}$ and other variables, we can find the ancestors and children of $X_{i}$. In detail, for each variable $X_{j}$ in ${X}_\mathbf{par(i)}$, we regress $X_{i}$ on $X_{j}$ and test whether the residual is independent of $X_{j}$. At the same time, we regress $X_{j}$ on $X_{i}$ and test the independence between the residual and $X_{i}$. Then according to Lemma \ref{pro:pair}, we can determine whether $X_{j}$ is an ancestor or child of $X_{i}$, or whether there is a confounder for them. 

If we have determined some parents or children for measured variable $X_{i}$, we can remove the common cause for two measured variables that are adjacent with the determined causal relationship by regression \cite{shimizu2011directlingam}, and then perform the step as above. This can determine most of the undetermined causal relations that are not influenced by confounders.  

\subsubsection{Identification of causal direction between variables not directly influenced by the same confounder}

After identifying the unconfounded ancestor, some cases where the causal structure between the measured variables cannot be identified because of the indirect latent confounders. They contain two cases:
\begin{enumerate}
    \item The parent and children of the measured variable $X_{i}$ are directly influenced by the same latent confounder $L_{j}$, while $X_{i}$ is not adjacent to (or equivalently, not directly influenced by) $L_{j}$; 
    \item Two or more parents of the measured variable $X_{i}$ are influenced by the same latent confounder $L_{j}$, while $X_{i}$ is not adjacent to $L_{j}$.
\end{enumerate}

\textbf{Case 1:} For the first case, and using Figure \ref{figure.exp2}(a) as an example, $X_1$ and $X_3$ are directly influenced by the hidden common cause $L_1$, but $X_2$ is not. The PAG obtained by Stage I is shown in \ref{figure.exp2}(b). Then for any of the three pairs of the three variables $X_1$, $X_2$, and $X_3$, regression is performed, and the independence of the residuals and the predictor variable is tested. But we can only determine $X_1 \rightarrow X_2$, and cannot identify $X_2-X_3$. If we can remove the indirect cause of $X_2$, then $X_2-X_3$ can be determined. After determining that $X_1 \rightarrow X_2$, we regress $X_2$ on $X_1$ and replace $X_2$ with its corresponding residual $R_{2,1}$. We can find that if the causal relationship between $R_{2,1}$ and $X_3$ also satisfies model (\ref{eq:model}), we can use Lemma \ref{pro:pair} to determine $X_2 \rightarrow X_3$. Next, we generalize Lemma 2 proposed by Shimizu et al. \cite{shimizu2011directlingam} to the latent confounder case and call it Lemma \ref{pro:residualModel}.

\begin{figure}[htbp]
	\centering
	\includegraphics[width=0.5\textwidth]{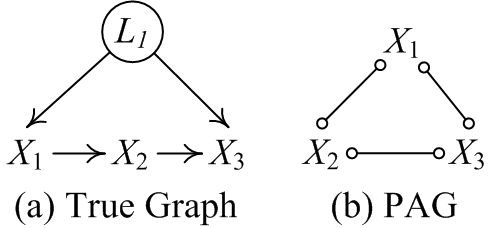} 
	\caption{An example that cannot be identified by Lemma \ref{pro:pair}: (a) a true causal graph; (b) the PAG estimated by Stage I using the data generated according to (a). In these graphs, $X_{1},X_{2}$, and $X_{3}$ represent the measured variables while $L_{1}$ represents a latent confounder.} \label{figure.exp2}
\end{figure}

\begin{lem}
    Assume that the data over measured variables $\mathbf{X}$ follows Model (\ref{eq:model}). Let $\mathbf{X}_{\mathbf{Pa}_{i}}$ denote a set of all found parents of $X_{i}$ ($X_{i} \in \mathbf{X}$) and $\mathbf{R}$ be the result of replacing each $X_{i}$ with its residual from regressing on $\mathbf{X}_{\mathbf{Pa}_{i}}$. Then, an analog of Model (\ref{eq:model}) holds as follows: $\mathbf{R} = \mathbf{B}_{R}\mathbf{R} + \Lambda\mathbf{L} + \mathbf{E}_{R}$, where $\mathbf{B}_{R}$ is a matrix of causal strengths among the residuals that corresponds to the measured variables, $\Lambda$ is a matrix of causal influences of the latent confounders $\mathbf{L}$ on measured variables, and the noise terms in $\mathbf{E}_{R}$ are mutually independent and non-Gaussian. 
    \label{pro:residualModel}
\end{lem}

\begin{proof}
	Without loss of generality, we assume that $\mathbf{B}$ in Equation \ref{eq:model2} can be permuted to a strictly lower triangular matrix. Therefore, $\mathbf{A}$ of Equation \ref{eq:model2} is also a lower triangular matrix with diagonal entries. Since $X_{\mathbf{Pa}_{i}}$ is the parent of $X_{i}$ for each $X_{i}$, $\mathbf{A}_{i, \mathbf{Pa}_{i}}$ is equal to the regression coefficient obtained by linear regression of $X_{i}$ on $X_{\mathbf{Pa}_{i}}$. Therefore, through linear regression, the causal effect of $X_{\mathbf{Pa}_{i}}$ on $X_{i}$ is removed from $X_{i}$, that is, each $A_{i,j}$ in $\mathbf{A}_{i, \mathbf{Pa}_{i}}$ is 0, and $X_{\mathbf{Pa}_{i}}$ does not influence the residual $\mathbf{R}_{i,\mathbf{Pa}_{i}}$. Therefore, for $\mathbf{R}$, its corresponding $\tilde{\mathbf{A}}$ is still a strictly lower triangular matrix, (i.e., $\tilde{\mathbf{B}}$ is also a strictly lower triangular matrix). Therefore, $\mathbf{R} = \mathbf{B}_{R}\mathbf{R} + \Lambda\mathbf{L} + \mathbf{E}_{R}$ holds.  
\end{proof}

Thus, for each variable $X_{i}$, after removing the effect of all determined parents of $X_{i}$ by regressions and independence tests we can find the parents and children of $X_{i}$. The details of the procedure are as follows.

First, for each pair of measured variables $X_{i}$ and $X_{j}$, we perform a linear regression of $X_{i}$ on $X_{j}$, and test whether the corresponding residual $R_{i,j}$ is independent of $X_{j}$. If it is, we orient $X_{j} \to X_{i}$. Otherwise, we test whether the reverse causal direction is accepted. If neither of them is accepted, there may be at least one latent confounder or a common ancestor influencing them. After refining some edges, we remove the effects of parents by regressing the variable on its determined parents and using the corresponding residuals to replace the variables. This is because if $X_{i}$ and $X_{j}$ are unconfounded, then after we remove the information in $X_{i}$ and $X_{j}$ that can be explained by their common ancestors, the residuals in $X_{i}$ and $X_{j}$ are unconfounded and they admit the same causal direction as that between $X_{i}$ and $X_{j}$. Then, we iterate the first step for the variables with an undetermined edge between them to determine more edges, until no independence between a potential parent of variable and the corresponding residual is accepted.

\textbf{Case 2:} We then consider the second case for $X_{i}$ and its parents $X_{j}$ and $X_{k}$; we still cannot determine the causal relationship between $X_{i}$ and $X_{j}$ or between $X_{i}$ and $X_{k}$. That is to say, $X_{j}$ and $X_{k}$ are mediating variables for the path from $L$ and $X_{i}$ so that $L$ is a common cause of $X_{i}$ and $X_{j}$, and of $X_{i}$ and $X_{k}$. Using $X_{k}$ to ``block" this path will remove the influence from $L$ to $X_{i}$. This inspired us to apply regressions to address the problem, with the following theorem confirming its correctness.

\begin{lem}
    Suppose that the data over variables $\mathbf{X}$ were generated by Equation \ref{eq:model} and assumptions $\mathbf{A1}-\mathbf{A3}$ hold. Let $X_{\mathbf{par}(i)}$ denote a set of measured variables that are potential parents of $X_{i}$ and $X_{\mathbf{S}_{i}} \subset X_{\mathbf{par}(i)}$. Let $R_{i,\mathbf{S}_{i}}$ be the residual of regressing $X_{i}$ on $X_{\mathbf{S}_{i}}$. In the limit of infinite data, $X_j$ is an unconfounded parent of $X_{i}$, if and only if there exist a subset $X_{\mathbf{S}_{i}}$, defined above such that $X_{j}$ is independent of $R_{i,\mathbf{S}_{i}}$. 
	\label{pro:child}
\end{lem}
\begin{proof}
    Note that $X_{\mathbf{par}(i)}$ denote all potential parents of $X_{i}$ and $X_{\mathbf{S}_{i}}$ be a subset of $X_{\mathbf{par}(i)}$. If variable $X_{j}$ is in $X_{\mathbf{par}(i)}$, then $X_{j}$ might be a (confounded) parent or child of $X_{i}$, or there is a latent confounder between $X_{i}$ and $X_{j}$ without directed edge. 
    
    1. Consider that $X_{j}$ is a parent of $X_{i}$ and there is no latent confounder between $X_{i}$ and $X_{j}$. First, we can rewrite Equation \ref{eq:model} as
    
	\begin{equation*}
	\begin{aligned}
	\mathbf{X} = \left(\begin{array}{c}
	X_{\mathbf{S}_{i}} \\ X_{i} 
	\end{array}\right) &= (\mathbf{I}-\mathbf{B})^{-1} (\Lambda \mathbf{L} + \mathbf{E})= \mathbf{C} (\Lambda \mathbf{L} + \mathbf{E})\\
	&= \left[\begin{array}{cc}
	\mathbf{C}_{\mathbf{S}_{i}} & \mathbf{C}_{\mathbf{S}_{i},i}^{T} \\ \mathbf{C}_{i,\mathbf{S}_{i}} & 1
	\end{array}\right] \left[ \begin{array}{c}
	\Lambda_{X_{\mathbf{S}_{i}}}\mathbf{L} + E_{\mathbf{S}_{i}} \\ \Lambda_{i} \mathbf{L} +E_{i}\end{array}\right],
	\end{aligned}
	\end{equation*}
	where $\mathbf{C}=(\mathbf{I}-\mathbf{B})^{-1}$. The inverse of $\mathbf{C}$ can be written as
	\begin{equation}
	\mathbf{C}^{-1} =\left[\begin{array}{cc} \mathbf{D}^{-1} & -\mathbf{C}_{\mathbf{S}_{i},i}^{T} \mathbf{D}^{-1}\\
	-\mathbf{D}^{-1}\mathbf{C}_{i,\mathbf{S}_{i}} & (1-\mathbf{C}_{\mathbf{S}_{i},i}^{T}\mathbf{C}_{\mathbf{S}_{i}}^{-1}\mathbf{C}_{i,\mathbf{S}_{i}}) ^{-1}\end{array}\right],
	\end{equation}
	where $\mathbf{D}= \mathbf{C}_{\mathbf{S}_{i}} - \mathbf{C}_{\mathbf{S}_{i},i}^{T} \mathbf{C}_{i,\mathbf{S}_{i}}$. Thus, $1-\mathbf{C}_{\mathbf{S}_{i},i}^{T}\mathbf{C}_{\mathbf{S}_{i}}^{-1}\mathbf{C}_{i,\mathbf{S}_{i}}=1$.
	
	Then, regressing $X_{i}$ on $X_{\mathbf{S}_{i}}$, we have
	\begin{equation*}
	\begin{aligned}
	R_{i,\mathbf{S}_{i}} 
	=& X_{i} - \frac{\mathbb{E}[{X_{i},X_{\mathbf{S}_{i}}}]}{\mathbb{E}[{X_{\mathbf{S}_{i}}}^2]} X_{\mathbf{S}_{i}}\\
	=& \mathbf{C}_{i,\mathbf{S}_{i}} (\Lambda_{X_{\mathbf{S}_{i}}}\mathbf{L} + \mathbf{E}_{\mathbf{S}_{i}}) + (\Lambda_{i} \mathbf{L} +E_{i} ) \\
	&- \alpha_{i,\mathbf{S}_{i}} ( \mathbf{C}_{\mathbf{S}_{i}} (\Lambda_{X_{\mathbf{S}_{i}}}\mathbf{L} + E_{\mathbf{S}_{i}}) +\mathbf{C}_{\mathbf{S}_{i},i}^{T} (\Lambda_{i} \mathbf{L} +E_{i})  )\\
	=& \{\mathbf{C}_{i,\mathbf{S}_{i}} \Lambda_{X_{\mathbf{S}_{i}}}+\Lambda_{i}
	- \alpha_{i,\mathbf{S}_{i}}(\mathbf{C}_{\mathbf{S}_{i},i}\Lambda_{X_{\mathbf{S}_{i}}}+\mathbf{C}_{\mathbf{S}_{i},i}^{T} \Lambda_{i} ) \}\mathbf{L} \\
	&+ (\mathbf{C}_{i,\mathbf{S}_{i}} - \alpha_{i,\mathbf{S}_{i}} \mathbf{C}_{\mathbf{S}_{i}}) \mathbf{E}_{\mathbf{S}_{i}} + (1-\alpha_{i,\mathbf{S}_{i}} \mathbf{C}_{\mathbf{S}_{i},i}^{T}) E_{i},		
	\end{aligned}
\end{equation*}
	where $\alpha_{i,\mathbf{S}_{i}} = \frac{\mathbb{E}[{X_{i},X_{\mathbf{S}_{i}}}]}{\mathbb{E}[{X_{\mathbf{S}_{i}}}^2]}$.
	
	Thus, the residual $R_{i,\mathbf{S}_{i}}$ will be a linear mixture of latent confounders, the noise terms of $X_{i}$ and all variables in $X_{\mathbf{S}_{i}}$. If the linear contributions of all variables in $X_{\mathbf{S}_{i}} \setminus X_{j}$ to the influence of $X_{j}$ on $X_{i}$ have been partialed out, that is, $\mathbf{C}_{i,\mathbf{S}_{i}} - \alpha_{i,\mathbf{S}_{i}} \mathbf{C}_{\mathbf{S}_{i}} = \mathbf{0}^{T}$, then we can obtain
	
	\begin{equation}
	\begin{aligned}
	R_{i,\mathbf{S}_{i}} &= \Lambda_{i} (1
	- \alpha_{i,\mathbf{S}_{i}}\mathbf{C}_{\mathbf{S}_{i},i}^{T}) \mathbf{L}+ (1-\alpha_{i,\mathbf{S}_{i}}\mathbf{C}_{\mathbf{S}_{i},i}^{T}) E_{i}\\
	&=\Lambda_{i} \mathbf{L}+  E_{i}.
	\label{eq:23}
	\end{aligned}
	\end{equation}
	
	Because there is no latent confounder between $X_{i}$ and $X_{j}$, the coefficient of $\mathbf{L}$ on $X_{j}$ is zero. Thus, from Equation \ref{eq:23}, $R_{i,\mathbf{S}_{i}}$ is independent of $X_{j}$ due to Theorem \ref{theo.1}.
    
    2. Consider that $X_{j}$ is a confounded parent or confounded child of $X_{i}$, or that there is a latent confounder between them without directed edge. The effect of the latent confounder may not vanish by multiple regression on any measured variables. So the residual of regressing $X_{i}$ on $X_{\mathbf{S}_{i}}$ ($X_{j}\in X_{\mathbf{S}_{i}}$) is dependent of $X_{j}$. 
    
    3. Consider that $X_{j}$ is a child of $X_{i}$ and there is no latent confounder between $X_{i}$ and $X_{j}$. If we regress $X_{i}$ on every $X_{\mathbf{S}_{i}}$ which contains $X_{j}$, the residual $R_{i,\mathbf{S}_{i}}$ will be a linear mixture of the noise term of $X_{i}$ and others. According to the Equation \ref{eq:model}, $R_{i,\mathbf{S}_{i}}$ is a linear mixture of the noise term of $X_{i}$, $X_{j}$ and others. Thus, $X_{j}$ is dependent of $R_{i,\mathbf{S}_{i}}$.
\end{proof}

Lemma \ref{pro:child} inspires a method of identifying the local structure of measured variables for the second case by analyzing the PAG. According to Lemma \ref{pro:child}, we start by performing a multiple regression of undetermined variable $X_i$ on every subset of its potential parents to test whether there exist \textit{two} variables $X_{j}$ and $X_{k}$ such that the corresponding residual is independent of these two variables. If the independence holds for variable $X_j$ and the residual, then $X_j$ is a parent of $X_i$. Similarly, if undetermined edges remain, we perform a multiple regression on the subset of the potential parents containing \textit{three} variables and then \textit{four} variables, and so on, to find variables in the subset of potential parents that are unconfounded parents (according to independence tests) until no subset such that the residual is independent of the predictor(s) can be found.

Using these methods, we find local causal structures over measured variables that are adjacent to an undetermined edge in $\mathcal{G}$. In this stage, when an edge is reoriented, we apply FCI orientation rules \cite{zhang2008completeness} to orient other undetermined edges and update the corresponding potential parent sets. Using these orientation rules saves a number of regressions and independence tests.

As an example, using the causal graph from Figure \ref{figure.f1}(b), we obtain the output by (multiple) regressions and independence tests. By applying the FCI orientation rule $\mathcal{R}8$ \cite{zhang2008completeness}, we reorient the edge between $X_{2}$ and $X_{5}$ according to assumption \textbf{A3}. The final graph produced by this stage is shown in Figure \ref{figure.f1}(c).

According to the stage II process, the following theorem summarizes identifiability.
\begin{theo}
Suppose that the data over variables $\mathbf{X}$ was generated by model (\ref{eq:model}) and assumptions $\mathbf{A1}-\mathbf{A3}$ hold. Let $\mathcal{G}_1$ denote the output of stage I. The pairs of variables with an undirected edge in between in $\mathcal{G}_1$ that are not actually directly influenced by the same latent confounder are identified by stage II of FRITL.
\end{theo}

\begin{proof}
Under the assumptions of the theorem, stage I removes most of the independent causal edges, which provides stage II with (conditional) independence information. With the help of Lemmas \ref{pro:pair} and \ref{pro:residualModel}, we can determine the direction of the causal relationship between variables that are not directly influenced by the same latent confounder. Lemma \ref{pro:child} provides the identifiability conditions of the causal structure between observed variables that are not influenced by the same latent confounder.
\end{proof}

As a consequence, what remains to be identified is the causal structure between variables directly influenced by the same latent confounders.

\subsection{Stage III: Detecting Shared Latent Confounders Using the Triad Condition}

\begin{figure}[t]
	\centering
	\includegraphics[width=0.8\textwidth]{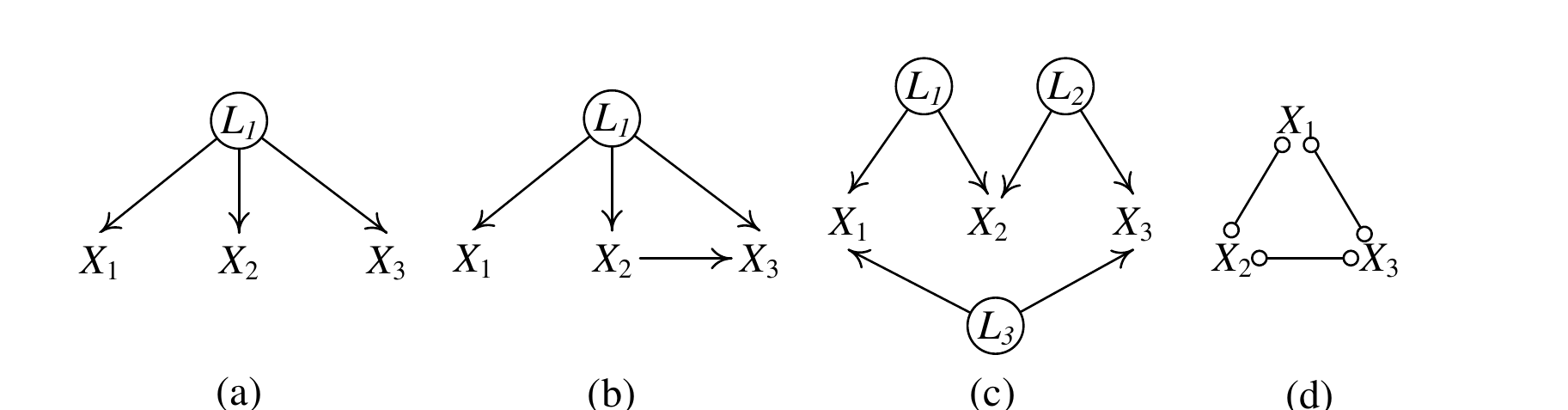}
	\caption{The three causal graphs in (a), (b), and (c) all correspond to the completed nondirected PAG obtained by FCI shown in (d). They cannot be distinguished by stage II of our method.} \label{figure.stage3PAG}
\end{figure}

The procedure so far determines whether latent confounders exist in many cases, but some graphs corresponding to the PAG shown in Figure \ref{figure.stage3PAG} remain undistinguishable. Stage II only considers two variables each time. Thus, there are no details about the causal relationship  (e.g., whether there is a direct causal relation and which way the causal influence goes) between two variables that are directly influenced by a same latent confounder, because these two variables both contain the information of the latent confounder. Suppose that assumptions \textbf{A3} hold. Interestingly, if we consider another measured variable at the same time, we can treat this third variable as a ``conditional variable" or an ``instrumental variable" and use it to help remove the indirect causal relationship (due to the existence of latent confounders) through the path containing latent confounders. The Triad condition \cite{cai2019Triad}, which the proposed procedure makes use of, is described as follows. 

\begin{defi}\textbf{(Triad condition)} Suppose assumptions \textbf{A1} - \textbf{A3} hold. For a triple of measured variables $(X_{i}, X_{j}, X_{k})$ generated by (\ref{eq:model}). $X_{j}$ and $X_{k}$ are Triad conditional on $X_{i}$ (or given $X_{i}$), when the residual $E_{X_{k}, X_{j} \mid X_{i}}=X_{k}-\frac{Cov(X_{i},X_{k})}{Cov{(X_{i},X_{j})}} \cdot X_{j}$ is independent of $X_{i}$, that is, $E_{X_{j}, X_{k} \mid X_{i}} \perp\!\!\!\perp X_{i} $. If the Triad condition is satisfied, we denote it by $Triad(X_{j}, X_{k}\mid X_{i})$.
	\label{def:triad}
\end{defi}

It is easy to establish the property that the Triad condition is symmetric, that is, $Triad(X_{j}, X_{k}\mid X_{i})$ if and only if $Triad(X_{k}, X_{j}\mid X_{i})$.

The three possible causal graphs given in Figures \ref{figure.stage3PAG} (a)-(c) over three measured variables $(X_{i},$ $X_{j}, X_{k})$ correspond to the PAG in Figure \ref{figure.stage3PAG} (d) that is produced by stage I. None of the three undirected edges can be reoriented by stage II. Based on the Triad condition, We detect whether three variables share a latent confounder via the following Theorem \ref{Triad1}.

\begin{theo}
Suppose that the data over variables $\mathbf{X}$ was generated according to Equation \ref{eq:model} and assumptions $\mathbf{A1}-\mathbf{A4}$ hold. Let $\mathcal{G}_2$ denote the output of stage II of FRITL. For three observed variables $(X_{i},$ $X_{j}, X_{k})$ with an undetermined edge between each pair of them in $\mathcal{G}_2$, if and only if three Triad conditions hold among $(X_{i},$ $X_{j}, X_{k})$, then $(X_{i},$ $X_{j}, X_{k})$ are directly influenced by a same latent confounder, and each pair of observed variables are not directly connected.
	\label{Triad1}
\end{theo}

\begin{proof}
	Suppose that the data over variables $\mathbf{X}$ was generated by Equation \ref{eq:model}. Without loss of generality, we assume three variables in $\mathbf{X}$, $X_{i}$, $X_{j}$ and $X_{k}$, are standardized (they have a zero mean and a unit variance) and have causal relations in between, in addition to the influences of latent confounders. Note that if a coefficient is zero, then the corresponding edge vanishes. Then we have
	\begin{equation}
	\begin{aligned}
	X_i &= \Lambda_{i}\mathbf{L}+E_{i},\\
	X_j &= b_{j,i}X_i + \Lambda_{j}\mathbf{L} +E_{j},\\
	X_k &= b_{k,i}X_i + b_{k,j}X_j + \Lambda_{k}\mathbf{L} +E_{k}.
	\end{aligned}
	\end{equation}
	
	Three kinds of Triad conditions might hold among three variables: $Triad(X_{j},X_{k}\mid X_{i})$, $Triad(X_{i},X_{k}\mid X_{j})$ and $Triad(X_{i},X_{j}\mid X_{k})$. So we consider three cases conditioning on different variables as follows.
	
	1. Considering a Triad condition conditioning on $X_{i}$, we can obtain the following reference variable 
	\begin{equation*}
	\begin{aligned}
	E_{X_{k}, X_{j} \mid X_{i}}&= X_{k}-\frac{\mathrm{Cov}(X_{i},X_{k})}{\mathrm{Cov}(X_{i},X_{j})}\cdot X_{j}\\
	&= (b_{k,i}X_i + b_{k,j}X_j + \Lambda_{k}\mathbf{L} +E_k) 
	 - \frac{b_{k,i}+ \Lambda_{i}\Lambda_{k}^T}{b_{j,i}+\Lambda_{i}\Lambda_{j}^T} \cdot (b_{j,i}X_i + \Lambda_{j}\mathbf{L} +E_j)\\
	&= \frac{ b_{j,i}\Lambda_{k}-b_{k,i}\Lambda_{j}}{b_{j,i}+\Lambda_{i}\Lambda_{j}^T} \cdot \mathbf{L} + \frac{ (b_{k,i}\Lambda_{i}\Lambda_{j}^T - b_{j,i}\Lambda_{i}\Lambda_{k}^T)}{b_{j,i}+\Lambda_{i}\Lambda_{j}^T} \cdot X_{i}  - \frac{b_{k,i}+\Lambda_{i}\Lambda_{k}^T}{b_{j,i}+\Lambda_{i}\Lambda_{j}^T} \cdot E_{j}+ E_{k}\\
	&= \frac{\Lambda_{i}\Lambda_{i}^T \cdot (b_{k,i}\Lambda_{j} - b_{j,i}\Lambda_{k})}{b_{j,i}+\Lambda_{i}\Lambda_{j}^T} \cdot \mathbf{L} + \frac{ (b_{k,i}\Lambda_{i}\Lambda_{j}^T-b_{j,i}\Lambda_{i}\Lambda_{k}^T)}{b_{j,i}+\Lambda_{i}\Lambda_{j}^T} \cdot E_{i} \\
	&- \frac{\Lambda_{i}\Lambda_{k}^T+b_{k,i}}{b_{j,i}+\Lambda_{i}\Lambda_{j}^T} \cdot E_{j}+ E_{k},
	\end{aligned}
	\end{equation*}
	which is a linear mixture of independent variables, namely, $\mathbf{L}$, $E_{i}$, $E_{j}$, and $E_{k}$. As we know, $X_{i}$ is a mixture of independent variables $\mathbf{L}$ and $E_{i}$. If the parameters in this model are not zero, it is dependent on $X_{i}$ because of Theorem \ref{theo.1}. 
	Next, if it satisfies $Triad(X_{j}, X_{k}\mid X_{i})$, i.e., $E_{X_{k}, X_{j} \mid X_{i}}$ is independent of $X_{i}$. According to Theorem \ref{theo.1}, at most one of the coefficients of their common parameters, $\mathbf{L}$ and $E_{i}$ should be zero. Therefore, $\frac{ \Lambda_{i}\Lambda_{i}^T \cdot (b_{k,i}\Lambda_{j} - b_{j,i}\Lambda_{k})}{b_{j,i}+\Lambda_{i}^T\Lambda_{j}}$ and $b_{k,i}\Lambda_{i}\Lambda_{j}^T-b_{j,i}\Lambda_{i}\Lambda_{k}^T$ should be equal to zero, i.e., $b_{j,i}=b_{k,i}=0$, because $\Lambda_{i}$, $\Lambda_{j}$ and $\Lambda_{k}$ are nonzero. Then, $E_{X_{k}, X_{j} \mid X_{i}}$ becomes a linear mixture of $E_{j}$, and $E_{k}$ and is independent of $X_{i}$. Thus, there are no edges between $X_{i}$ and $X_{j}$, and between $X_{k}$ and $X_{i}$.
	
	2. Consider a Triad condition conditioning on $X_{j}$, we can obtain the following reference variable 
	\begin{equation*}
	\begin{aligned} 
	E_{X_{k}, X_{i} \mid X_{j}}=& X_{k}-\frac{\mathrm{Cov}(X_{k},X_{j})}{\mathrm{Cov}(X_{i},X_{j})}\cdot X_{i} \\
	=& (b_{k,i}X_i + b_{k,j}X_j + \Lambda_{k}\mathbf{L} +E_{k})- (b_{k,i} + \frac{\Lambda_{k}\cdot (\Lambda_{j}^T + b_{j,i}\Lambda_{i}^T)+ b_{k,j}}{b_{j,i}+\Lambda_{i}\Lambda_{j}^T}) \cdot X_{i}\\
	= &(\Lambda_{k}\cdot \mathbf{L} + b_{k,j}\cdot (b_{j,i}X_i + \Lambda_{j}\mathbf{L} +E_{j}) +E_{k}) - \frac{\Lambda_{k}\cdot (\Lambda_{j}+b_{j,i}\Lambda_{i}) + b_{k,j}}{b_{j,i}+\Lambda_{i}\Lambda_{j}^T} \cdot X_{i}\\
	=& (b_{k,j}b_{j,i}\Lambda_{i}+b_{k,j}\Lambda_{j}+\Lambda_{k}-\frac{(\Lambda_{k} \Lambda_{j}^T+b_{j,i}\Lambda_{k}\Lambda_{i}^T + b_{k,j})\Lambda_{i}}{b_{j,i}+\Lambda_{i}\Lambda_{j}^T}) \cdot \mathbf{L} \\
	& -(b_{k,j}b_{j,i}-\frac{\Lambda_{j}\Lambda_{k}^T+b_{j,i}\Lambda_{i}\Lambda_{k}^T   +b_{k,j}}{b_{j,i}+\Lambda_{i}\Lambda_{j}^T} )\cdot E_{i} + b_{k,j} \cdot E_{j} + E_{k},
	\end{aligned}
	\end{equation*}
	which is a linear mixture of four independent variables, namely, $\mathbf{L}$, $E_{i}$, $E_{j}$, and $E_{k}$. We can see that 
	\begin{equation}
	X_{j} = b_{j,i}X_i + \Lambda_{j}\mathbf{L} +E_{j} = (b_{j,i}\Lambda_{i} + \Lambda_{j}) \cdot \mathbf{L} + b_{j,i} \cdot E_{i} + E_{j},
	\end{equation}
	which is a mixture of three independent variables $\mathbf{L}$, $E_{i}$ and $E_{j}$. If all parameters in this model are non-zero, $E_{X_{k}, X_{i} \mid X_{j}}$ is dependent of $X_{j}$ because of the Theorem \ref{theo.1}. 
	
	If all three variables are directly influenced by the same latent confounder, satisfies $Triad(X_{i}, X_{k}\mid X_{j})$, i.e., $E_{X_{k}, X_{i} \mid X_{j}}$ is independent of $X_{j}$. According to Theorem \ref{theo.1}, at most one of the coefficients on their common parameters, $\mathbf{L}$, $E_{i}$ and $E_{j}$, should be zero. Therefore, $b_{k,j}$ would be zero, and then we can see that $b_{j,i}$ would be zero, too. Then, $E_{X_{k}, X_{i} \mid X_{j}}$ becomes a linear mixture of $E_{i}$ and $E_{k}$, and is independent of $X_{j}$. This also shows that the graph in which there is at most one directed edge between two measured variables and one latent confounder influences them at the same time are distinguishable by the Triad condition.
	
	3. Consider a Triad conditioning on $X_{k}$, similar with the two cases above, we can know that if there is only one edge between $X_{i}$ and $X_{j}$, i.e., $b_{k,j}=b_{k,i}=0$, then the graph implies Triad condition $Triad(X_{i}, X_{j}\mid X_{k})$.
	
	In conclusion, if there is not a directed edge between any pair of measured variables, that is, $b_{j,i}=b_{k,j}=b_{k,i}=0$, then the corresponding causal graph implies three Triad conditions, which are $Triad(X_{j}, X_{k}\mid X_{i})$, $Triad(X_{i}, X_{k}\mid X_{j})$, $Triad(X_{i}, X_{j}\mid X_{k})$. According to assumption $\textbf{A3}$, $b_{j,i}=0$ means that there is no direct edge between observed variables $X_{i}$ and $X_{j}$. If at least two causal strengths in $\{b_{j,i},b_{k,j},b_{k,i}\}$ are zero, then the causal structure over $(X_{i}, X_{j}, X_{k})$ satisfied one of the graph given in Figure \ref{figure.f3}. For three variables that are mutually adjacent with undetermined edges in the sub-graph of $\mathcal{G}_2$, if there exist one observed variable that is not adjacent to other two observed variables, then stage III of FRITL is able to identify these observed variables are influenced by the same latent confounder.
\end{proof}

Theorem \ref{Triad1} provides a criterion for detecting the latent confounders that directly influence three (or more) measured variables. Therefore, let $\mathcal{G}_{2}$ be the graph produced by stage II. We test Triad conditions on every triple of variables with undetermined edges among them in $\mathcal{G}_{2}$, and determine whether these three variables are influenced by a same latent confounder according to Theorem \ref{Triad1}. For example, in Figure \ref{figure.f1}(c), the triple $X_{5}$, $X_{7}$, $X_{8}$ satisfies three Triad conditions. We remove the edges among $X_{5}$, $X_{7}$ and $X_{8}$ and record that they are influenced by the same latent confounder. The output of stage III is given as Figure \ref{figure.f1}(d). After this stage, we can group some variables that are directly influenced by the same latent confounders, and eliminate more undetermined edges.

From the proof of Theorem \ref{Triad1}, if only one Triad condition is satisfied among three variables with undetermined edges in between, then the undetermined sub-graph might be one of the four cases given in Figure \ref{figure.f3}. With the help of Triad conditions, we can further apply overcomplete ICA to select the best model if needed.

\begin{figure}[h!t]
	\centering
	\includegraphics[width=0.8\textwidth]{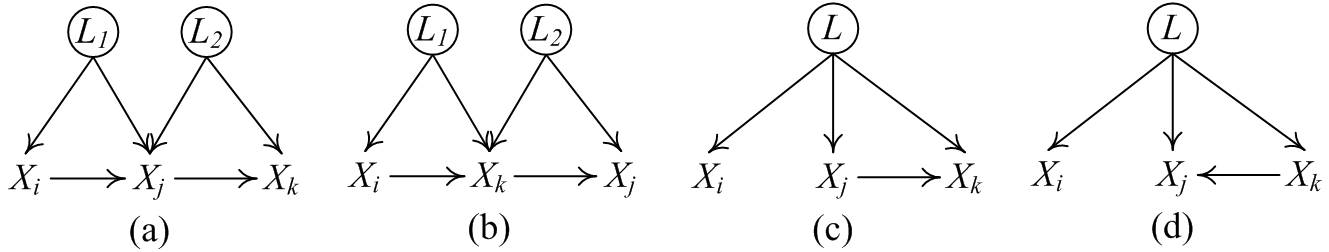}
	\caption{Four causal structures corresponding to the PAG given in Figure \ref{figure.stage3PAG} and satisfying the Triad condition $Triad(X_{j}, X_{k}\mid X_{i})$.} \label{figure.f3}
\end{figure}

If we cannot find any Triad conditions among the considered triple of variables, then the original causal graph belongs to one of the following cases:

\begin{itemize}
    \item There are two or more than two directed edges between the three measured variables in the original causal graph, in which these variables are directly influenced by the same latent confounder;
    \item Each pair of the three measured variables is directly influenced by a latent confounder, as in Figure \ref{figure.stage3PAG} (c).
\end{itemize}

\subsection{Stage IV: Estimating Remaining Undetermined Local Structures Using Overcomplete ICA}
In the general case, some undetermined edges not identified by the previous stages might remain. For example, in Figure \ref{figure.f1}(d), the edge between $X_{4}$ and $X_{7}$ cannot be determined by the preceding stages. We now apply overcomplete ICA locally to identify the local undetermined causal structures, that is, using overcomplete ICA only on the data of $X_{4}$ and $X_{7}$ after regressing their known parents out in order to remove the common causal effects.

According to Hoyer et al. \cite{hoyer2008estimation}, two latent variable LiNGAM models are observationally equivalent if and only if the distribution $P$ of the observed data is identical for these two models. A latent variable LiNGAM model, where each latent variable is a root node (i.e., has no parents) and has at least two children (direct descendants), is a canonical model.
Under assumption \textbf{A3}, we note that $\mathbf{A}$ in Equation \ref{eq:model2} can be estimated up to the permutation and scaling of the columns, as given in the following lemma. 

\begin{lem}
	If assumptions \textbf{A1}-\textbf{A4} are true, and $\mathbf{X}$ is generated according to (\ref{eq:model2}), $\mathbf{A}$ is identifiable up to permutation and scaling of columns. All the causal structure is identified up to observationally equivalent canonical models.	
\end{lem}\label{le:oICA}

\begin{proof}
	This lemma is implied by Theorem 10.3.1 in \cite{kagan1973characterization} or Theorem 1 in \cite{eriksson2004identifiability}. It is also proven in \cite{hoyer2008estimation}.	
\end{proof}

Let $\mathcal{G}_{3}$ be the graph obtained by stage III with undetermined edges. If $\mathcal{G}_{3}$ has many variables with undetermined edges in between, applying overcomplete ICA on all of them together has a very high procedural complexity with limited estimation accuracy. 
Stage II determines all unconfounded edges, and the variables with undetermined edges in between are directly influenced by latent confounders.  
We notice that if several measured variables are directly influenced by the same confounder, then a clique forms in the output of Stage II, with an underdetermined edge between each pair of them. Stage III, with output $\mathcal{G}_{3}$, already identifies a special case where multiple variables share the same confounder by checking for the Triad condition. Taking this further, we must figure out whether the variables
are directly influenced by the same latent confounders. To do so, we consider the subsets of the variables forming a maximal clique involving only undetermined edges and then apply overcomplete ICA to
estimate their causal structure. We understand that it is not necessary for the variables in the same maximal clique to be directly influenced by the same confounder, as seen in the structure given in Figure \ref{figure.stage3PAG} (c). 

 Specifically, we consider only the undetermined edges in $\mathcal{G}_{3}$, and separate relevant variables into different (possibly overlapping) maximal cliques with all edges undetermined in each. For variables in the same maximal clique, we estimate the causal structure together with the influences from confounders. For each maximal clique, it is possible for all variables in the maximal clique to be directly influenced by a few or even one latent confounder. As a consequence, we apply overcomplete ICA on each complete maximal clique, and the number of latent confounders can be estimated by model selection, if needed.

\begin{figure}[htbp]
	\centering
	\includegraphics[width=0.8\textwidth]{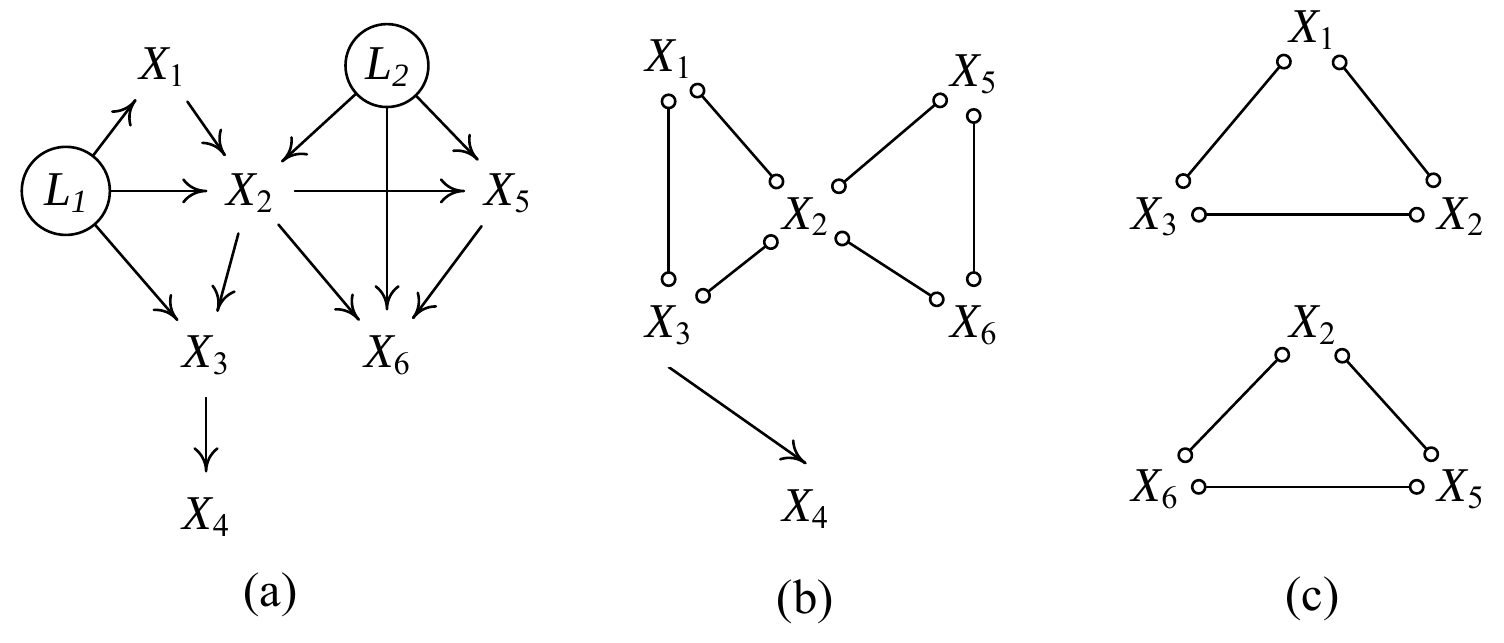}
	\caption{An example for stage IV: (a) true causal graph; (b) a PAG corresponding to (a), produced by stage III; and (c) two undetermined maximal cliques separated from (b).} \label{figure.ica}
\end{figure} 

Considering the example in Figure \ref{figure.ica}, we apply the first three stages of our method over the data generated by Figure \ref{figure.ica}(a) to obtain Figure \ref{figure.ica}(b). 
We then apply overcomplete ICA on the variables in the two maximal cliques (with undetermined edges), given in Figure \ref{figure.ica}(c),  separately, recovering the two local causal structures. 

\section{Discussion}
In this section, we show how the first three stages of our method make overcomplete ICA more accurate and efficient.

Consider a true causal structure among one of the graphs shown in Figure \ref{figure.f3} that satisfies $Triad(X_{j}, X_{k}\mid X_{i})$. In practice, exclusively using overcomplete ICA to discover the causal structure without applying stages I–III first, the algorithm needs to iterate all possible causal structures with permutation and scaling of columns of $\mathbf{A}$ to find the best model. Moreover, because the number of latent confounders $N_{L}$ is unknown, we need to test all cases based on all possible numbers of $N_{L}$. As the number of measured variables increases, more cases need to be computed, with a greater probability of falling into a local maximum. In contrast, stages I and II of our method return only the causal graph given in \ref{figure.stage3PAG} (d), but the use of Triad condition in Stage III determines that the causal structures must be among the graphs given in Figure \ref{figure.f3}.We need to perform overcomplete ICA only on these four possible causal structures, returning only the best one rather than estimating the causal structures in a large space with different numbers of latent confounders and possible causal graphs. In this case, the use of Triad condition greatly reduces the search space and identifies possible sub-structure informed by the previous result, as shown in Figure \ref{figure.stage3PAG} (a). 

Hence, we can determine the unconfounded structures by stage II, and further determine some sub-structures over three (or more) observed variables with latent confounders by stage III. These stages not only make the FCI result more informative, but can reduce the search space when using overcomplete ICA if needed. Further, our method is less prone to local optima and more efficient.

\section{Experiments}

In this section, we conduct simulation experiments and apply our method to real-world data to evaluate our method’s performance.

\subsection{Synthetic data}

We performed simulations as follows. We randomly generated causal structures over measured variables and latent confounders with different average $indegree =$ 0.5, 1, 1.5, 2, 2.5, 3, 3.5, which are the ratios of the number of indegree edges to the number of measured variables. Each causal structure had $10$ measured variables. In each generated graph, we randomly designed different ratios of latent confounders on the number of measured variables, $p = 0.1, 0.2, 0.3, 0.4, 0.5$. Data for these variables were not given to the search procedure. The maximum number of children for each latent confounder was $3$. Based on each causal structure, we generated the data according to LvLiNGAM, with the causal strength between different variables randomly chosen in the range of $(-1,-0.2] \cup [0.2,1)$ and the noise term for each variable randomly chosen from the uniform distribution on the interval $[-0.5,0.5]$. In addition, our generated data consisted of $1000$ samples in each set. For each setting, we repeated the algorithm 50 times, each time randomly generating a causal graph and coefficient, and then sampling a data set. 

In these experiments, we used the FCI Java implementation from the Tetrad \footnote{\href{www.phil.cmu.edu/tetrad/}{www.phil.cmu.edu/tetrad/}}for stage I of our method. Pseudo-code for FCI is described in Spirtes et al. \cite{spirtes2001causation}. For the regression and independence test, we used least squares regression to perform linear regressions and the kernel-based conditional independence (KCI) test \cite{zhang2011kernel} to conduct (conditional) independence tests between variables. Here we used 0.05 as the significance level for the independence test. In these experiments, we evaluated the performance of our method in terms of arrowheads among measured variables of recovered causal graphs and pairs of measure variables that were detected to be directly influenced bylatent confounders, by computing precision, recall, and F1 score. Precision is the percentage of correct causal edges between measured variables among all causal edges returned by the algorithm. Recall is the percentage of correct causal edges that are found by the search among true causal edges between measured variables. The F1 score is defined as
\begin{equation}
    F1 score = \frac{2 \times precision \times recall}{precision+recall}.
\end{equation}

To show the performance of different stages of our method, we also applied FCI (stage I), FRI (the combination of stages I and stage II), FRIT (the combination of first three stages), and FRITL (all stages) in the generated data sets. To show the performance of our framework, we used PairwiselvLiNGAM \cite{entner2010discovering} as the second phase of the framework, calling the new method FCI-pw. Besides, we utilized ParceLiNGAM \cite{tashiro2014parcelingam} that can be against the latent confounders as another compared method, and DirectLiNGAM \cite{shimizu2011directlingam} that assumes causal sufficiency to evaluate whether the performance of algorithm considering latent confounders is better.

Figure \ref{figure.arrowheads} gives the performance of our methods for the arrowheads. In this figure, the y-axis is precision, recall, or F1 score; a higher value means higher accuracy.

\begin{figure}[ht]
	\centering
	\includegraphics[width=0.98\textwidth]{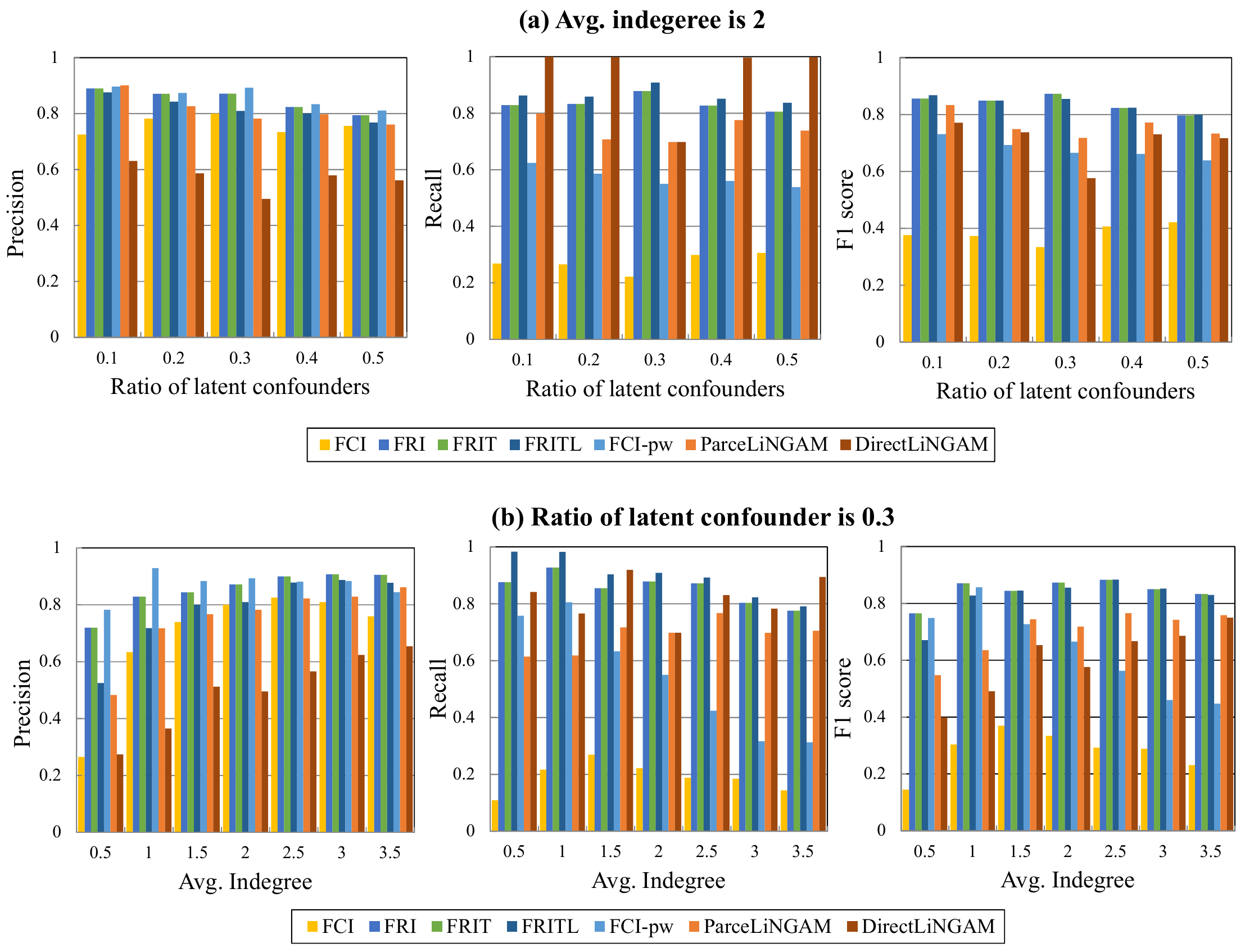}
	\caption{The evaluation of arrowheads among measured variables of recovered causal graphs.} \label{figure.arrowheads}
\end{figure}

\paragraph{Sensitivity of different settings in arrowheads among measured variables.}
From Figure \ref{figure.arrowheads}, we can see that the performance of FRI and FRIT are the same. This is because the stage III of our algorithm does not refine the causal directions between measured variables. The precision of FCI is mostly higher than 0.7, while the recall is smaller than 0.4,indicating that causal directions found by FCI are correct but with many undetermined causal edges remaining. Our method works better than FCI in both precision and recall, indicating the presence of more information. In all cases, FRI’s precision was higher than that of FRITL while FRI’s recall was lower than that of FRITL. This indicates that stage IV of our algorithm - with overcomplete ICA technique, correctly estimate some undetermined edges but introduce some redundant edges, which shows this technique is not reliable. The results of FRI and FCI-pw show nearly identical precision when the causal graph has a sparseness of 2, with FRI having higher recall. This shows that the causal directions found by pairwiselvLiNGAM are nearly correct, but most confounded edges, both latent and observed, could not be determined. Our method successfully finds confounded edges or indirect latent confounded edges by comparison. Compared with ParceLiNGAM, our method performs better, even when the ratio of the latent confounder is 5, further showing the effectiveness of our algorithm. Our method’s higher F1 score compared to DirectLiNGAM  reflects the importance of taking latent confounders into account. Figure \ref{figure.arrowheads} (a) shows our method to be mostly stable in the precision of arrowheads, which means stages II and III of our method are good for determining the causal direction of the PAG produced by FCI. In Figure \ref{figure.arrowheads} (b) shows that the recall of our method decreases as the graph density increases, with precision increasing. This means stages II and III of our method does not determine many causal edges between measured variables influenced by latent confounders. In summary, these results demonstrate the correctness and effectiveness of our method, with an ability to use a number of stages appropriate to the data if needed.

\begin{figure}[h]
	\centering
	\includegraphics[width=0.8\textwidth]{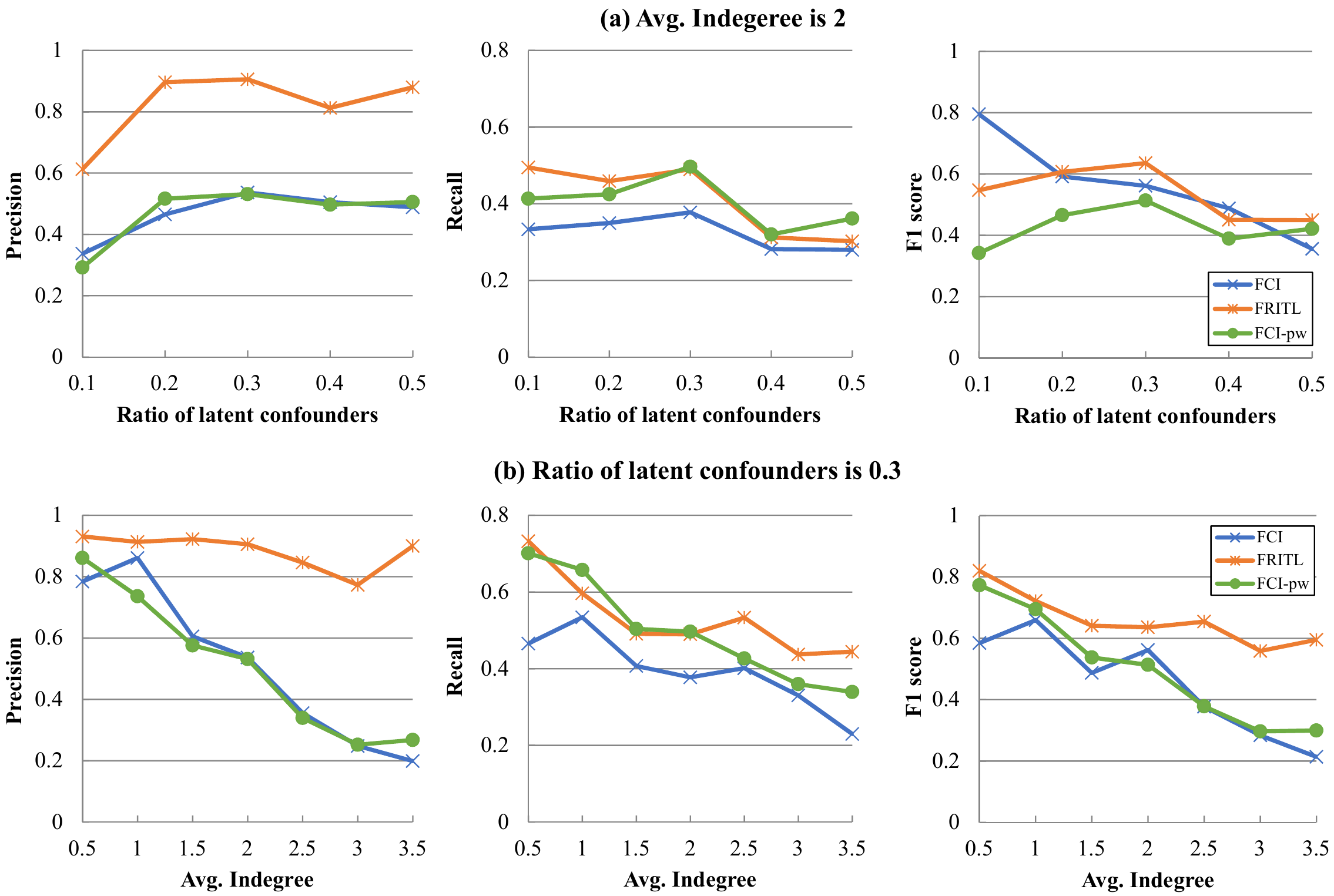}
	\caption{The evaluation of the pairs of measure variables detected to be directly influenced by latent confounders.} \label{figure.latent}
\end{figure}

\paragraph{Sensitivity of different settings in the pairs of measure variables that are detected to be directly influenced by latent confounders.} 
Our method also determines latent confounder influences which measured variables. Because ParceLiNGAM and DirectLiNGAM cannot do so, we only evaluate the performance of our method on the pairs of measure variables that are detected to be directly influenced by latent confounders, and compare them with the results produced by FCI and FCI-pw. As a reminder, FCI-pw only determines the unconfounded pairs of measured variables, so we treat the undetermined pair as directly influenced by the same latent confounder.As before, we evaluated the results according to the precision, recall, and F1 score. The results are given in Figure \ref{figure.latent}. Figure \ref{figure.latent}(a) shows that when the ratio of latent confounders is larger than $0.2$, the precision of our method is larger than $0.8$ whereas those that of FCI and FCI-pw are only around $0.6$. The recall of FCI-pw is higher than that of our method in most cases. This indicates FCI-pw does not find pairs not influenced by latent confounders, and the ones that our algorithm found are correct. Figure \ref{figure.latent}(b) shows that when the number of latent confounders increases, the precision of method’s recovery of the causal relationships between observed and latent variables lightly increases slightly increases whereas that of the results learned by FCI and FCI-pw decreases, which illustrates that our method finds more true latent confounders and recovers more causal relations between latent confounders and measured variables even in a dense graph. Although the recall of our algorithm is lower than that of FCI-pw when the average indegree of the causal graph is $1$ and $1.5$, the $F_1$ score of our algorithm is higher. All of these figures show that our methods correctly recover most measured variables that are directly influenced by latent confounders. 

\subsection{FMRI Task Data}
To test the performance of our method in a real problem, we applied our method to real functional magnetic resonance imaging (fMRI) task data previously published in \cite{ramsey2010six}. These data sets consist of 9 subjects that are judged whether a pair of visual stimuli rhymed or not. Data was acquired with a 3T scanner, with $TR = 2$ seconds, so the sample size for each subject is $160$ \cite{sanchez2019estimating}. Raw data is available at the OpenfMRI Project\footnote{\href{https://openfmri.org/dataset/ds000003/}{https://openfmri.org/dataset/ds000003/}}; and the preprocessed data used here is available\footnote{\href{https://github.com/cabal-cmu/Feedback-Discovery}{https://github.com/cabal-cmu/Feedback-Discovery}}. We use the preprocessed data in this experiment.

In these data sets, each subject has 9 variables, which are one input variable (Input) and eight regions of interest (ROIs). The Input variable is built by convolving the rhyming task boxcar model with a canonical hemodynamic response function. The eight ROIs are left and right occipital cortex (LOCC, ROCC); left and right anterior cingulate cortex (LACC, RACC); left and right inferior frontal gyrus (LIFG, RIFG); left and right inferior parietal (LIPL, RIPL). 

We applied our method and FCI on the concatenated data set for the Input variable and the eight regions of interest. This data set is combined by 1 repetition of 9 standardized individual data sets. Accordingly the sample size of this data set is $1440$. The threshold for the independence tests in this experiment is $0.01$. Figure \ref{figure.realG} gives the output graphs produced by FRITL and FCI, respectively.

\begin{figure}[htbp]
	\centering
	\includegraphics[width=0.5\textwidth]{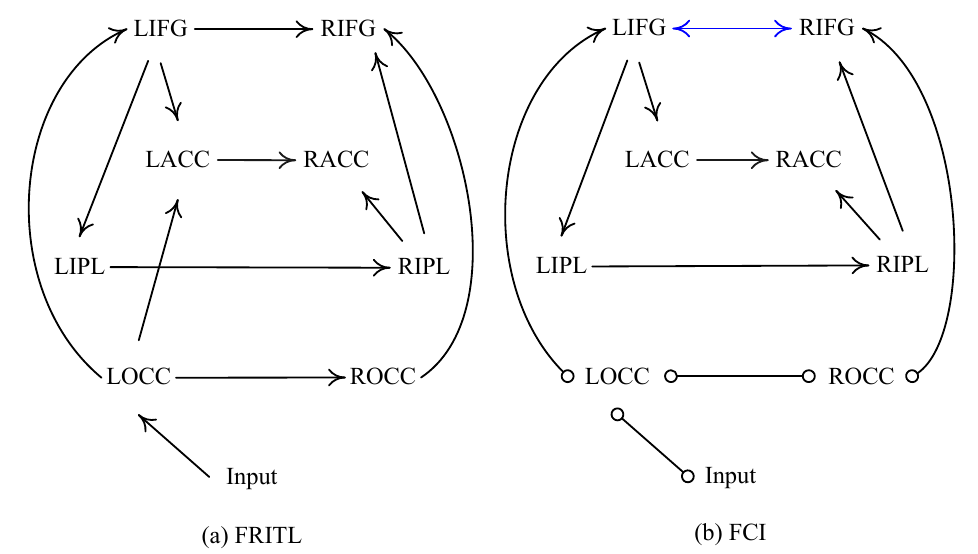}
	\caption{The output graphs produced by different methods on concatenated data set for eight bilateral regions of interest and one Input variable. The blue line denotes that there is a latent confounder for the adjacent measured variables.} \label{figure.realG}
\end{figure}

Compared fMRI task data with synthetic data, a shared common view is that stimulus input should go to the left occipital, feed upwards, and feed from left to right. Thus, the Input variables should be the exogenous variable in the true graph. From Figure \ref{figure.realG}, we can see that FRITL correctly outputs the edges from the Input variable to the region of interest, while FCI correctly outputs the adjacency on these edges but did not orient them. Figure \ref{figure.realG}(b) shows that there are four undetermined edges in the output of FCI, which are refined by our method (Figure \ref{figure.realG}(a)). It proves that the output of FCI is less informative than that of our method. From the results, we also show that the ROIs in the left hemisphere always be the causes of that in the right hemisphere, which is consistent with the common view.

\subsection{Sachs Data}
We also applied FRITL, FCI-pw, FCI, ParceLiNGAM, and DirectLiNGAM to Sachs data \cite{sachs2005causal}. Sachs data records many cellular protein concentrations in single cells. This data contains 9 files with varying interventions. In this experiment, we use these intervention knowledges so that the data contains 11 measured variables and 7466 samples. Figure \ref{figure.sachs} shows the results of FRITL, FCI, ParceLiNGAM, and DirectLiNGAM. We also visualize the ground-truth given in Figure \ref{figure.sachs} for evaluation. The results show that FCI-pw did not refine any edges from the FCI result, while FRITL reorientated two edges and located the latent confounders. ParceLiNGAM failed to find the causal order of 9 variables except for Mek and Raf, meaning it could not determine the causal relationships among most of the variables. it performs pruning, DirectLiNGAM still output still outputs many redundant causal edges. This is because the presence of latent confounder for two observed variables introduced false causal relationships among these observed variables. This demonstrates the importance of considering the existence of latent confounders. Comparing FRITL output and the ground truth, the causal graph estimated by our method does not contain the $Mek\to Erk$ edge, which is well-established.

\begin{figure}[ht]
	\centering
	\includegraphics[width=0.99\textwidth]{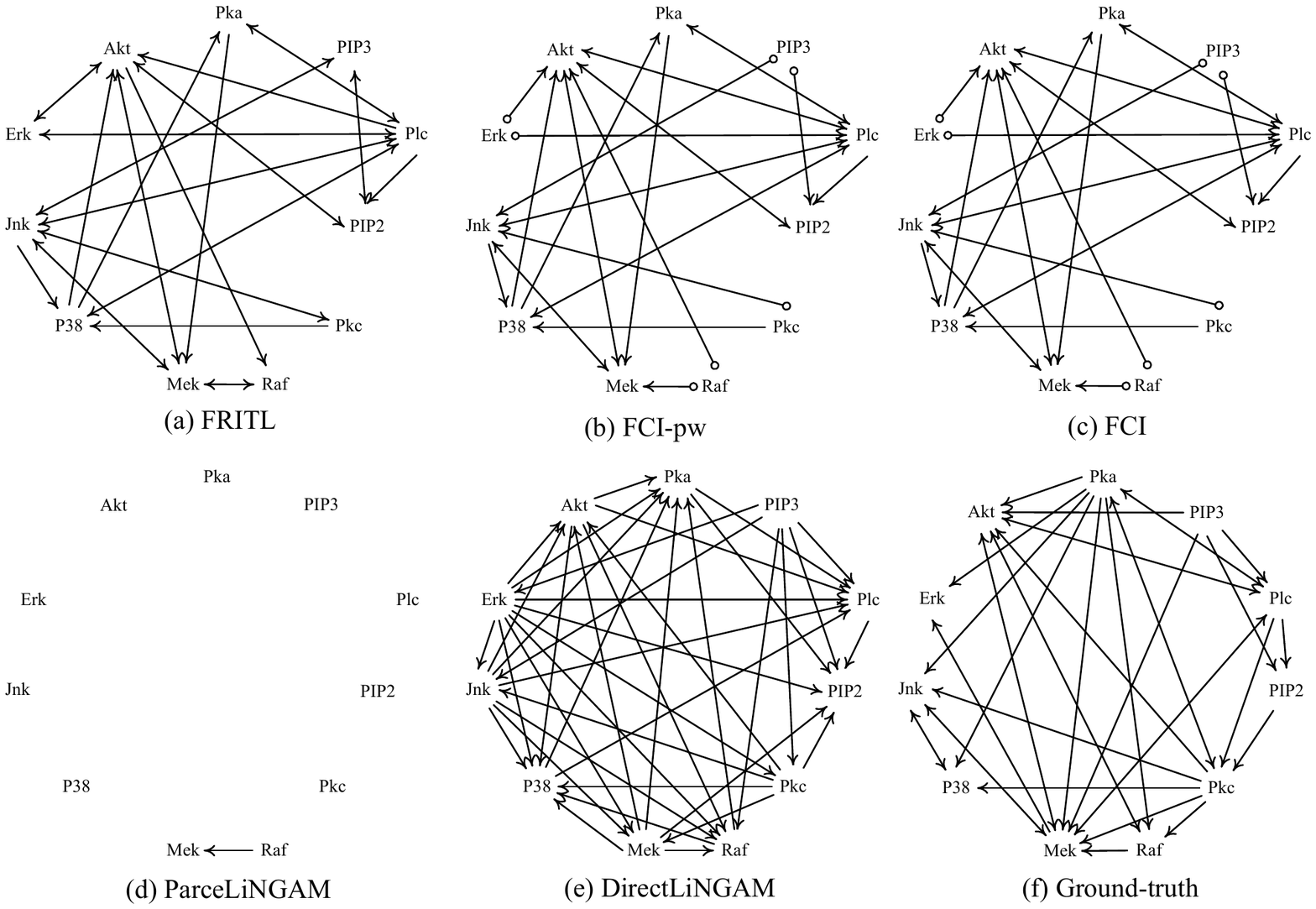}
	\caption{(a) - (e) The results of different methods applied in Sachs data. (f) The ground truth.} \label{figure.sachs}
\end{figure}

\section{Conclusions}

In this paper, we provide a hybrid method to reconstruct a causal graph of observed variables and latent confounders. In the proposed framework, we use FCI to decompose the global structure and use the independence noise condition, Triad condition, and overcomplete ICA to infer remaining local structures. The simulated experiments results show that FRITL is asymptotically correct and is more informative than FCI. In application to real functional magnetic resonance data and Sachs data, FRITL yields results in good agreement with neuropsychological opinion and precise agreement a causal relation known from the experimental design. In the future, we would like to generalize this framework to nonlinear cases.


\acks{Wei Chen would like to thank the CABAL group at Carnegie Mellon University for insightful discussion and suggestions. This research was supported in part by National Natural Science Foundation of China (61876043), Science and Technology Planning Project of Guangzhou (201902010058). Wei Chen would like to gratefully acknowledge the financial support from China Scholarship Council (CSC) and the Outstanding Young Scientific Research Talents International Cultivation Project Fund of Department of Education of Guangdong Province.}


\vskip 0.2in

\bibliography{FRITL}

\end{document}